\documentclass[final]{ecai}

\usepackage{latexsym}
\usepackage{amssymb}
\usepackage{amsmath}
\usepackage{amsthm}
\usepackage{booktabs}
\usepackage{enumitem}
\usepackage{graphicx}
\usepackage{color}
\usepackage{soul}
\usepackage{tikz}
\usepackage{pifont}
\usepackage{amssymb}
\usepackage{multirow}
\usepackage{bm}
\usetikzlibrary{matrix, arrows.meta, positioning, shapes.geometric}

\newcommand{\BibTeX}{B\kern-.05em{\sc i\kern-.025em b}\kern-.08em\TeX}

\newcommand{\cmark}{\checkmark}

\begin{document}

\begin{frontmatter}



\title{ScoresActivation: A New Activation Function for Model Agnostic Global Explainability by Design}




\author[A]{\fnms{Emanuel}~\snm{Covaci}}
\author[A]{\fnms{Fabian}~\snm{Galiș}\footnotemark}
\author[B]{\fnms{Radu}~\snm{Balan}}
\author[A]{\fnms{Daniela}~\snm{Zaharie}}
\author[A]{\fnms{Darian}~\snm{Onchis}}

\address[A]{West University of Timișoara, Romania \\ \small \texttt{emanuel.covaci98@e-uvt.ro, fabian.galis00@e-uvt.ro, daniela.zaharie@e-uvt.ro, darian.onchis@e-uvt.ro}}

\address[B]{University of Maryland, United States\\ \small \texttt{rvbalan@umd.edu}}

\begin{abstract}

Understanding the decision of large deep learning models is a critical challenge for building transparent and trustworthy systems. Although the current post hoc explanation methods offer valuable insights into feature importance, they are inherently disconnected from the model training process, limiting their faithfulness and utility.
In this work, we introduce a novel differentiable approach to global explainability by design, integrating feature importance estimation directly into model training. Central to our method is the ScoresActivation function, a feature-ranking mechanism embedded within the learning pipeline. This integration enables models to prioritize features according to their contribution to predictive performance in a differentiable and end-to-end trainable manner.
Evaluations across benchmark datasets show that our approach yields globally faithful, stable feature rankings aligned with SHAP values and ground-truth feature importance, while maintaining high predictive performance. Moreover, feature scoring is 150 times faster than the classical SHAP method, requiring only 2 seconds during training compared to SHAP's 300 seconds for feature ranking in the same configuration. Our method also improves classification accuracy by 11.24\% with 10 features (5 relevant) and 29.33\% with 16 features (5 relevant, 11 irrelevant), demonstrating robustness to irrelevant inputs. This work bridges the gap between model accuracy and interpretability, offering a scalable framework for inherently explainable machine learning.

\end{abstract}

\end{frontmatter}

\section{Introduction}

Modern deep learning models, including transformer-based architectures \cite{vaswani2017attention} and foundation models beyond transformers, operate as black boxes, that rely on the training of billions of parameters to provide high precision, at the cost of interpretability. In sensitive domains, such as healthcare and finance, understanding why a model makes a decision is as important as the decision itself. Accordingly, the field of Explainable AI (XAI) \cite{DoshiVelez2017} has produced numerous post-hoc explanation techniques that analyze a trained model with the purpose of identifying and ranking important features for the decision of the model. Most prominently, LIME (Local Interpretable Model-agnostic Explanations) \cite{Ribeiro2016} at local level and SHAP (SHapley Additive exPlanations) \cite{Lundberg2017} at both local and global level are popular methods that approximate the behavior of a model by assigning feature importance scores to individual predictions.
Gradient-based approaches similarly attribute an output to input features by accumulating gradients \cite{Sundararajan2017}. Although widely used, these methods are fundamentally disconnected from model training because they treat the model as fixed and provide explanations after the model is trained. The disconnection effect between training and explanation can lead to issues of faithfulness and reliability, especially when noise and irrelevant features are concerned.

A post-hoc explanation may appear plausible without truly reflecting the model’s decision mechanics. In fact, it has been demonstrated that certain explanation techniques can remain unchanged even when the model parameters are randomized, indicating they were not faithfully using the model’s internal logic \cite{Adebayo2018}. Similarly, attention weights in Transformer models are sometimes used as importance indicators, but it has been argued that “attention is not explanation” unless carefully validated \cite{Jain2019}. These findings highlight that conventional post-hoc explanations, under standard training, can be unfaithful or misleading \cite{Adebayo2018}. Moreover, because they are separate from training, they cannot influence the model to be more interpretable or stable.

One promising direction is to design models that are interpretable by design, integrating the explanation mechanism into the model’s architecture or training process \cite{Wang2019Designing}. Prior works have suggested training models to be “right for the right reasons” by adding explanation-based loss constraints \cite{Ross2017} or learning separate explanation networks in parallel \cite{Chen2018}. However, such approaches often require additional supervision or complex training regimes. There have also been calls to abandon post-hoc interpretation in favor of inherently interpretable models for high-stakes decisions \cite{Rudin2019}. Additive models and attention-based networks are examples that provide some insight into feature effects, but they either sacrifice predictive power or lack guaranteed faithfulness.

In this paper, we propose a novel approach that bridges the gap between model training and explanation. We introduce a new activation function we called ScoresActivation unit, emphasizing its use with either a classifier or a regressor that integrates feature importance estimation directly into its forward pass. ScoresActivation learns to produce an importance score for each input feature as part of the prediction process. Unlike post-hoc methods, our approach yields feature importance values inherently, ensuring that the explanation is consistent with how the model makes decisions. By entwining explanation with training, we demonstrate that we can achieve both high predictive performance and reliable, faithful explanations without additional post-processing.

\begin{figure}[h]
\centering
\includegraphics[width=0.7\linewidth]{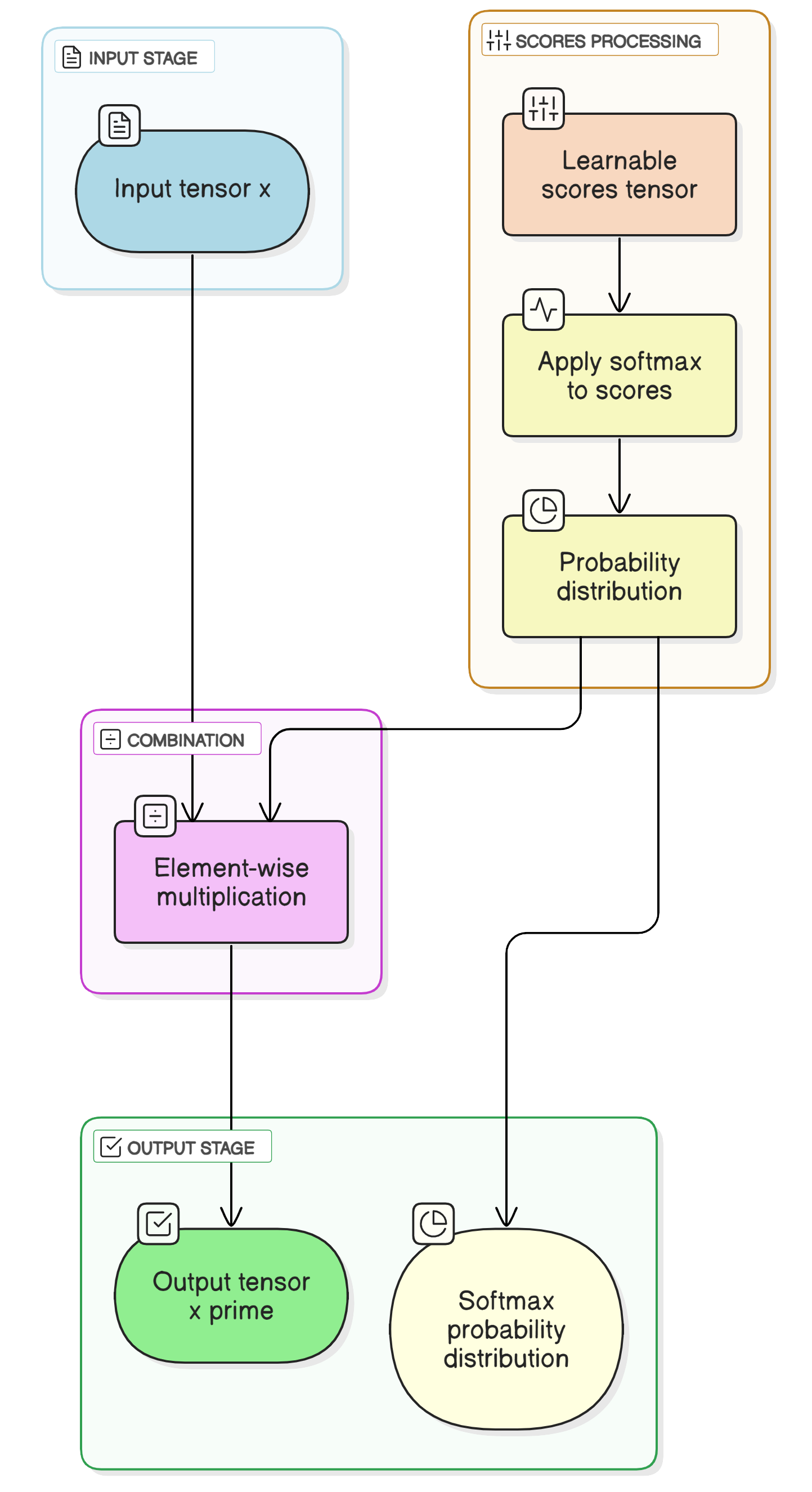}
\caption{Overview of the ScoresActivation function.}
\label{fig:overview}
\end{figure}

That means that our explainable model does not compromise performance, and even enhances it. Across several datasets, our model's accuracy is significantly better than the non-explainable baseline. We validate that the learned feature rankings from our model are faithful. They align closely with feature importances obtained via SHAP, which is a reliable post-hoc method. The model’s built-in explanation genuinely reflects the features that influence its predictions, which we confirm by high rank correlation with SHAP and the ground truth. 

Our approach provides massive computational benefits for explanation. 
This is ensured through the separation of features based on trainable scoring during model training, eliminating the need for additional post-processing, unlike the numerous SHAP post-hoc model evaluations.
Moreover, in our experiments, our model’s explanations are in average around 150\% 
faster to compute than SHAP’s, enabling near-real-time interpretation even on complex models. This efficiency makes our method practical for large-scale or time-sensitive applications. Moreover, the integrated nature of our model also contributes to training stability. 

We train the model with standard optimization, without the need for elaborate post-hoc analysis or separate explainer modules. The result is a simpler and more stable pipeline for achieving interpretability.


\section{Methodology}

We developed a model that outputs per-feature importance scores alongside its prediction. The importance mechanism is learned during training, tightly coupling the model’s predictive parameters with its explainability. For this purpose, we introduced \textbf{ScoresActivation}, a new architectural component that computes feature importance in a differentiable manner. This module enables the model to consider feature contributions in context (via Transformer self-attention) and produces a meaningful global importance ranking adjusted during training by incorporating each instance contribution. An overview of the function can be observed in Figure \ref{fig:overview}.




 
\subsection{ScoresActivation Function Definition}

Let $\mathbf{X} \in \mathbb{R}^{b \times d}$ be the input tensor, representing a batch of $b$ samples, each with $d$ features.  
Let the \textbf{scores vector} $\mathbf{s} \in \mathbb{R}^d$ be a learnable parameter vector of dimension $d$. The $i$-th component of this vector, $s_i$, represents the learned score for the $i$-th feature dimension.

The scores vector $s$ can be initialized using various strategies, such as starting with null values, employing a preliminary ground truth assumption, or assigning values randomly. However, as the experimental analysis suggests, the choice of initialization method has little to no impact on the subsequent process.

The function operates by using the learnable scores vector $\mathbf{s}$ to compute a set of weights that are applied to the input features. The process is as follows:


\noindent\textbf{Scores Normalization to Weights:}  
    The learned scores, represented by the components $s_i$ of the vector $\mathbf{s}$, are transformed into a normalized weight vector $\mathbf{w} \in \mathbb{R}^d$ via the softmax function:
    \begin{equation}\label{eq:ws}
        w_i^{(s)} = \frac{e^{s_i}}{\sum_{j=1}^{d} e^{s_j}}, \quad \text{for } i = 1, \ldots, d
    \end{equation}
    
This produces a probability distribution over the features, where higher scores result in higher weights. 
\\
\textbf{Input Feature Weighting:}  
    Each sample vector within the input tensor $\mathbf{X}$ is scaled element-wise by the weight vector $\mathbf{w^{(s)}}$. For the $k$-th input sample $\mathbf{x}_k \in \mathbb{R}^d$, the corresponding output vector is computed as:
    \begin{equation}
    \mathbf{y}_k = \mathbf{w^{(s)}} \odot \mathbf{x}_k
    \end{equation}
    where $\odot$ denotes the Hadamard (elementwise) product. This operation produces the scaled output tensor $\mathbf{Y} \in \mathbb{R}^{b \times d}$ across the entire batch.

The output tensor $\mathbf{Y}$ is used in subsequent computations, while the weight vector $\mathbf{w^{(s)}}$ is an intermediate construct derived from the learnable scores. The scores vector $\mathbf{s}$ is updated during training to optimize task-specific performance through backpropagation.



The main difference between a standard functional layer, characterized by a nonlinear activation function, i.e. $F(W\cdot x)$, and the ScoresActivation approach is that, as it is illustrated in Eq.~(\ref{eq:wdiag}) with $w_i^{(s)}$ given in Eq.~(\ref{eq:ws}), the output values of the layer depend linearly on $\mathbf{x}$ but nonlinearly on the scores $\mathbf{s}$. 

\begin{equation} \label{eq:wdiag}
S(x)=W\cdot diag(w_1^{(s)},\ldots, w_d^{(s)})\cdot x
\end{equation}

Looking deeper into the structure of the ScoreActivation transformation, we can see that 
it can be decomposed into two distinct components:

\paragraph{Nonlinear component:} The vector $w^{(s)}$ is applied via element-wise multiplication with the input, introducing nonlinearity if $w^{(s)}$ is obtained through a softmax transformation:
\begin{equation}
    w^{(s)} = \text{softmax}(s(\bm{x})),
\end{equation}
with $s(\bm{x})$ the learned scoring function. This ensures the weights are positive and sum to one.

\smallskip

\paragraph{Linear component:} The matrix $W$ then 
acts linearly by multiplication with 
the reweighted input:
\begin{equation}
    S(\bm{x}) = W \cdot (w^{(s)} \odot \bm{x}),
\end{equation}
where $\odot$ denotes the element-wise (Hadamard) product. 

This decomposition allows us to interpret $S(\bm{x})$ as a \emph{softmax-gated linear transformation} (illustrated in Figure \ref{fig:decomposition}), in which the softmax weights $w^{(s)}$ adaptively modulate the importance of each input dimension before the final projection. The overall score function thus combines:
a nonlinear, input-dependent gating mechanism, and a standard linear transformation via $W$.

In this paper, the focus is on using the scores to estimate the importance of the input features. Therefore, the ScoresActivation is applied at the first layer, but it can be incorporated at any hidden layer, as illustrated in Figure~\ref{fig:placement}. 

Concerning the impact on computational complexity of the ScoresActivation, it can be observed that when it is applied to a layer of $K$ units it introduces $K$ additional learnable parameters and the computation of the softmax transformation for a vector of $K$ elements. 





\begin{figure}[h]
\centering
\includegraphics[width=\linewidth]{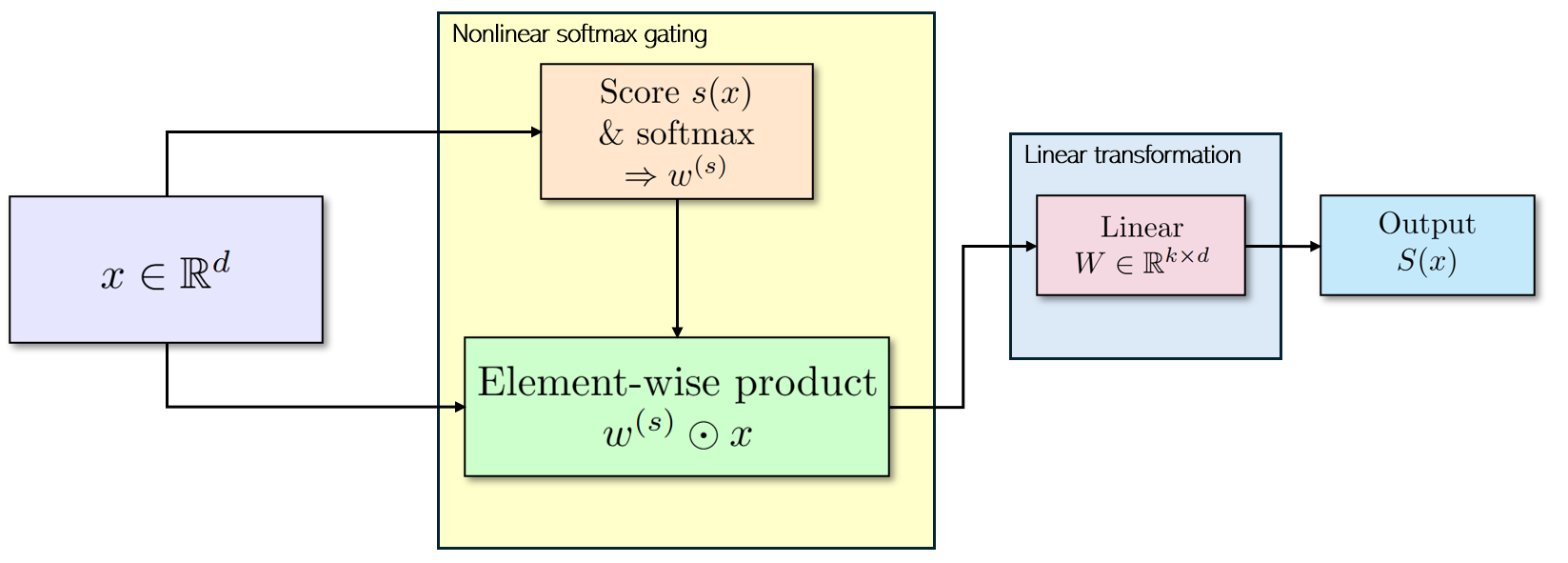}
\caption{Decomposition of $S(\bm{x}) = W \cdot \text{diag}(w^{(s)}) \cdot \bm{x}$ into a nonlinear softmax gating (with respect to the scores) and a linear transformation (with respect to the input vector).}
\label{fig:decomposition}
\end{figure}

\subsection{Theoretical Insights on ScoresActivation}

In this section, we provide some theoretical observations regarding the behavior of the proposed \textit{ScoresActivation} mechanism. Specifically, we discuss its differentiability, convergence under signal conditions, and robustness to irrelevant features.  
The insights provided are supported by our empirical observations and intuitive reasoning.

\smallskip
\noindent\textbf{Differentiable Feature Prioritization}. 
Let $f_\theta(x)$ denote a deep neural network with parameters $\theta$, and $x \in \mathbb{R}^d$ the input features. Let $S(\bm{x}) = \text{ScoresActivation}(\bm{x}) \in \mathbb{R}^d$ be a differentiable function that modulates feature importance. The network operates on modulated inputs $S(\bm{x}) \odot \bm{x}$, where $\odot$ denotes element-wise multiplication.

Assume the following properties:
\textbf{(P1)} $s_i(\bm{x}) \in (0, 1)$ for all features $i$; 
\textbf{(P2)} $S(\bm{x})$ is differentiable with respect to both $x$ and $\theta$;
\textbf{(P3)} A feature scoring loss $\mathcal{L}_{\text{score}}(s)$ encourages sparsity or alignment with utility

\newtheorem{claim}{Observation}

\begin{claim}[Gradient Alignment with Predictive Signal]
As $S(\bm{x})$ is optimized jointly with the task loss $\mathcal{L}_{\text{task}}$, gradients through $S(\bm{x})$ increasingly align with those of $\mathcal{L}_{\text{task}}$, thereby prioritizing features that reduce predictive error.
\end{claim}

\begin{proof}[Rationale]
The ScoresActivation module is embedded within the computational graph. Informative features that consistently contribute to reducing the loss receive stronger gradient signals, leading to reinforced activations over training iterations.

Let us compare the gradient components of the function corresponding to the output values of the first hidden layer (that contains $K$ units) in the case when a logistic activation function and the ScoresActivation are used.  When the logistic activation function, $f$, is used, the output of the $k$-th hidden unit is:
\begin{equation}
y_k=f(\sum_{i=1}^dw_{ik}x_i), \quad k=1\ldots K
\end{equation}
and the partial derivatives with respect to the learnable weights, $w_{ik}$, are:
\begin{equation} \label{eq:derivLogistic}
\frac{\partial y_k}{\partial w_{ik}}=y_k(1-y_k)x_i.
\end{equation}
In the case of the ScoresActivation, the output of the $k$-th hidden unit is:
\begin{equation}
y_k=\sum_{i=1}^dw_{ik}w_i^{(s)}x_i, \quad i=1\ldots d, k=1\ldots K
\end{equation}
leading to the following partial derivatives with respect to the two categories of learnable parameters, $w_{ik}$ (Eq.~\ref{eq:derivW}) and $s_l$ (Eq.~\ref{eq:derivWs}):
\begin{equation}\label{eq:derivW}
\frac{\partial y_k}{\partial w_{ik}}=w_i^{(s)}x_i, \quad i=1\ldots d, k=1\ldots K
\end{equation}
\begin{equation} \label{eq:derivWs}
\frac{\partial y_k}{\partial s_l}=\sum_{i=1}^d w_{ik}w_i^{(s)}(\delta_{il}-w_l^{(s)})x_i, \quad l=1\ldots d
\end{equation}
In Eq.~(\ref{eq:derivWs}), $\delta_{il}$ is $1$ if $i=l$ and $0$ otherwise. From the partial derivatives described in Eq.~(\ref{eq:derivLogistic}) and in Eq.~(\ref{eq:derivW}) it follows that the gradient-based adjustment of the weights $w_{ik}$ is obtained by multiplying $x_i$ with a factor in $(0,1)$ for both activation functions. However, there are some notable differences between the two cases: in the case of the logistic function the factor is based on the output of the hidden unit ($y_k$), while in the case of ScoresActivation it is based on the normalized value of the score, $w_i^{(s)}$, corresponding to the input value, $x_i$. This difference illustrates the active role of the scores in the adjustment of all weights.
\end{proof}

\begin{claim}[Suppression of Irrelevant Features]
Let $x = [x_{\text{rel}}, x_{\text{irr}}]$ be the input partitioned into relevant and irrelevant features. If irrelevant features have low mutual information with the target $y$, then under sparsity-regularized training, the ScoresActivation values for $x_{\text{irr}}$ converge toward zero:
\[
\mathbb{E}[s_i(x)] \to 0 \quad \text{for} \quad i \in \text{irrelevant}.
\]
\end{claim}


\begin{proof}[Rationale]
Irrelevant features introduce noise into the gradient flow. With regularization (e.g., entropy or $\ell_1$ penalties), the optimizer minimizes their contribution, effectively pruning them from the predictive pathway.
\end{proof}

\begin{claim}[Convergence Insight]
If there exists a feature subset $S^* \subset \{1, \ldots, d\}$ such that the Bayes-optimal predictor depends only on $x_{S^*}$, and if the model is sufficiently expressive, then training with ScoresActivation converges to emphasizing $S^*$.
\end{claim}

\begin{proof}[Rationale]
This holds under the assumptions that: \textbf{(A1)} The scoring loss $\mathcal{L}_{\text{score}}$ imposes sparsity (e.g., via entropy or KL divergence). \textbf{(A2)} The training objective is minimized using gradient descent. \textbf{(A3)} The dataset is sufficiently large to expose the relevant feature structure.
\end{proof}

\begin{claim}[Stability of Feature Rankings]
Let $s^{(1)}(x), s^{(2)}(x), \ldots, s^{(n)}(x)$ be the ScoresActivation outputs across $n$ model retrainings. The variance of feature ranks is significantly lower than that of SHAP values:
\[
\mathrm{Var}[\text{rank}(s^{(i)}(x))] < \mathrm{Var}[\text{rank}(\phi_{\text{SHAP}}^{(i)}(x))].
\]
\end{claim}

\begin{proof}[Empirical Evidence]
Since ScoresActivation is trained as part of the model pipeline, it captures the inductive biases of the model consistently across runs. Post hoc methods like SHAP are sensitive to parameter initialization and stochastic model variance, which leads to ranking instability.
\end{proof}

In summary, the theoretical insights presented indicate that the ScoresActivation mechanism effectively aligns feature importance scores with the task-specific gradient signal, thereby guiding the model to emphasize features that contribute most significantly to predictive performance. Through its integration in the training pipeline and the application of sparsity-inducing regularization, the method naturally suppresses irrelevant or noisy inputs without requiring any post hoc filtering. Moreover, under suitable data conditions and model expressiveness, ScoresActivation exhibits convergence behavior toward sparse, informative feature subsets, aligning closely with the ground truth. Importantly, due to its differentiable and end-to-end design, the approach produces stable and consistent feature rankings across retraining instances, offering a faithful and reproducible alternative to traditional explainability methods such as SHAP.

\begin{figure*}[h]
\centering
\includegraphics[width=0.65\textwidth]{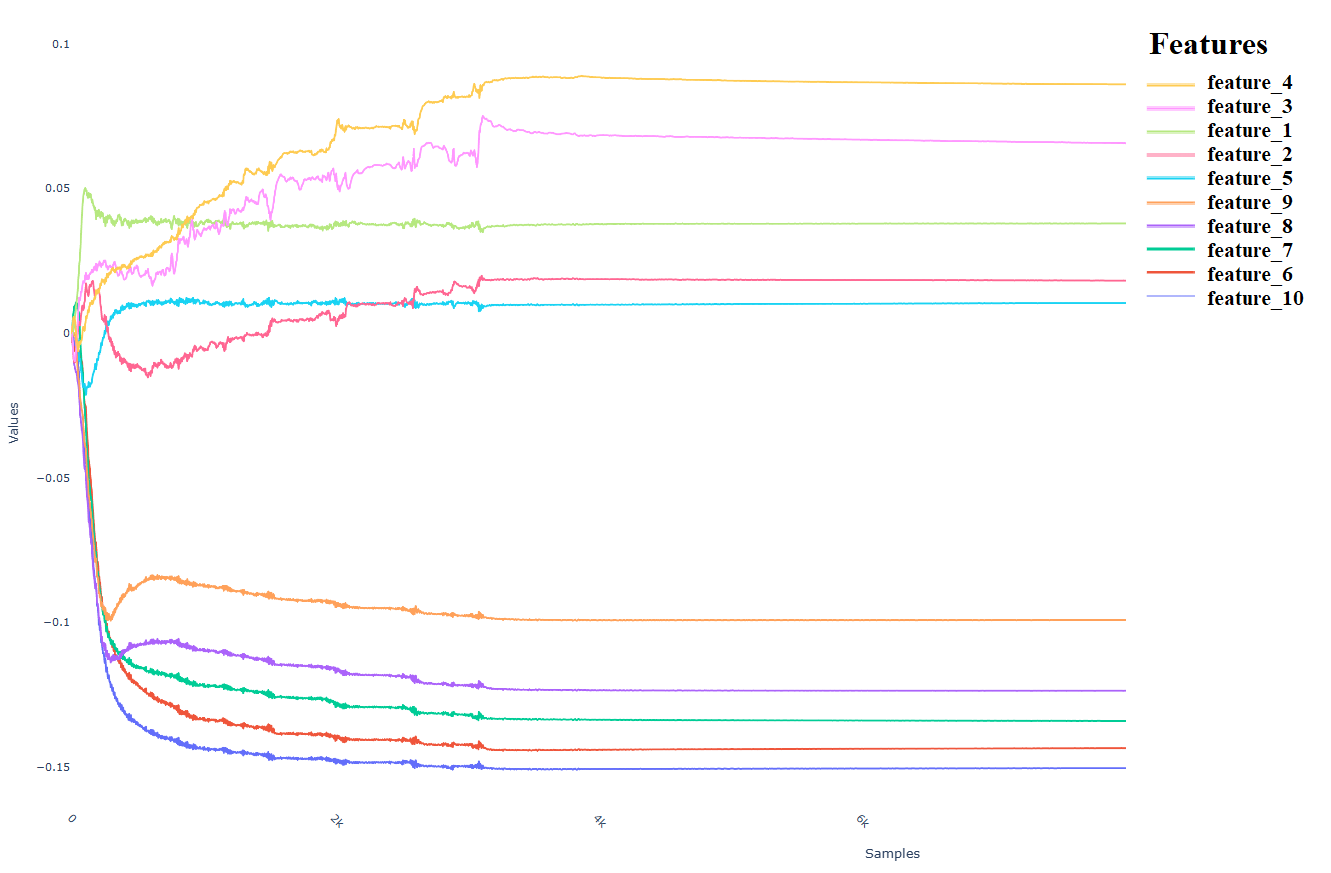}
\caption{Scores values separation of features after training with initialization of scores to $0$ (Synthetic dataset).}
\label{fig:scores_ranking_zero_init}
\end{figure*}

\begin{figure*}[]
\centering
\includegraphics[width=0.7\textwidth]{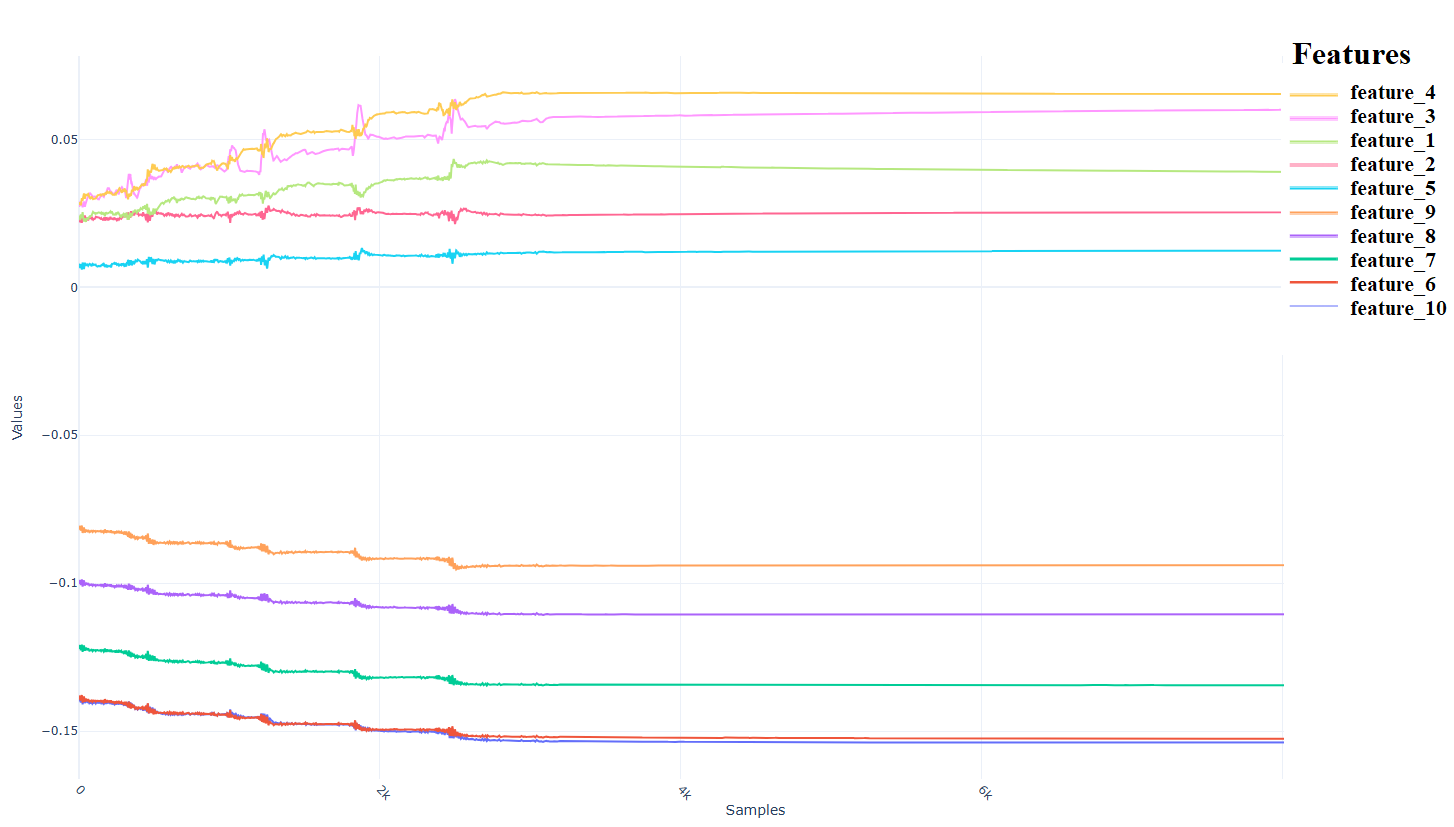}
\caption{Scores values after training with a ground truth initialization (Synthetic dataset).}
\label{fig:gt_init}
\end{figure*}

\section{Experiments}

%

We evaluated the model on both a newly generated synthetic dataset and well-known datasets (available in scikit-learn).\newline

\noindent\textbf{Synthetic Dataset.}
   A custom, synthetic dataset was constructed to facilitate systematic evaluation of binary classification algorithms and their feature selection capabilities. 
   Each instance 
   is a vector, of length $d$ 
   where $d$ 
   is the sum of a fixed number of \textit{\textbf{relevant}} features 
   (that are explicitly involved in the model) and a variable number of \textit{\textbf{irrelevant}} (noise) features. 
   To generate the dataset, the relevant features were drawn independently from a uniform distribution on $[4, 10]$ while the noise features were drawn from a uniform distribution on $[0, 1]$.
   For each instance, $\mathbf{x}_k=(x_{k,1},x_{k,2}, \ldots,x_{k,d})$, the target variable is determined 
   by a weighted sum of the first five features (the relevant ones), as specified in Eq.~\ref{eq:weightedsum}, followed by a threshold function (Eq.~\ref{eq:thresholdFunction}) leading to binary values.
\begin{equation} \label{eq:weightedsum}
z_k = 0.2 x_{k,1} + 0.3 x_{k,2} + 0.1 x_{k,3} + 0.05 x_{k,4} + 0.5 x_{k,5}
\end{equation}
\begin{equation}\label{eq:thresholdFunction}
y_k = 
\begin{cases}
1 & \text{if } z_k > 7.5 \\
0 & \text{otherwise}
\end{cases}
\end{equation}
The feature importance is explicitly determined by the coefficients of the linear combination in Eq.~\ref{eq:weightedsum}:
$x_5$ 
is most influential, followed by 
$x_2$, 
$x_1$, 
$x_1$, 
and $x_4$. 
The rest of the features do not influence the output values.\newline
   \noindent \textbf{Diabetes prediction dataset.}
   This dataset has $151$ patients with health status and lifestyle records. The dataset is used to explore the risk factors of diabetes. The dataset has $20$ features, which include demographic, clinical, and lifestyle information. Demographic information includes age, gender, and ethnicity. Clinical information includes BMI, waist measurement, blood pressure, blood glucose levels, HbA1c, and cholesterol. 
   There is also information on family history of diabetes and prior gestational diabetes to give a mix of genetic and behavioral risk factors. The combination of numerical and categorical data makes this dataset more valuable for exploring the relationship between lifestyle, clinical indicators, and diabetes.\newline
\noindent \textbf{Diabetes prediction dataset including synthetic random features for experimental analysis.}
   This dataset extends the classic diabetes prediction dataset by adding 
   some extra columns with random values. The purpose of this extension is to provide a controlled environment to test how machine learning models can distinguish between relevant and irrelevant features during the prediction process.
   In this extended version, $17$ new features are added. Each feature is generated independently using a uniform distribution on $[-100,100]$ 
   These features are completely unrelated to the presence or risk of diabetes and have no practical predictive value. By adding these irrelevant attributes, the extended dataset becomes a useful 
   resource to test feature selection techniques, model robustness, and to show the importance of not overfitting to spurious patterns. Researchers and practitioners can use this dataset to see how different algorithms and methods can focus on the clinically meaningful variables and ignore the noise introduced by the random features. 
   The extended diabetes prediction dataset is a practical way to test the challenges of feature relevance and model performance when faced with extraneous information.\newline
    \noindent \textbf{Breast Cancer Wisconsin} dataset is a broadly used dataset that contains 569 instances of tumor measurements with 30 numeric features describing characteristics of cell nuclei (radius, texture, perimeter, area, smoothness, concavity, etc., for mean, standard error, and “worst” or maximum values). The task is a binary classification of predicting whether a tumor is malignant or benign.\newline
    \noindent \textbf{Friedman 1 dataset.} 
     This dataset is described in the paper \cite{friedman1991multivariate}. In this model, the input features \( \mathbf{x} \) are independent random variables uniformly distributed between 0 and 1. The target output \( y \) is calculated using the following equation:
\begin{equation}
y(\mathbf{x}) = 10 \sin(\pi x_1 x_2) + 20 (x_3 - 0.5)^2 + 10 x_4 + 5 x_5 + \mathcal{N}(0, \sigma)
\end{equation} 
\begin{description}
    \item \( \mathbf{x} = (x_1, x_2, \dots, x_{10}) \) represents a vector of $10$ values sampled from independent random variables, uniformly distributed on $[0,1]$.
    \item The first five features, \( x_1, x_2, x_3, x_4, x_5 \), are used to calculate the output.
    \item \(\mathcal{N}(0, \sigma) \) denotes Gaussian noise with mean $0$ and standard deviation \( \sigma \).
\end{description}

This dataset is particularly valuable for testing feature selection techniques, as it involves a mixture of both relevant and irrelevant features.\newline
    \noindent \textbf{Friedman 2 dataset.} 
This is a synthetic dataset, 
that serves as a benchmark dataset 
for evaluating the performance of regression models, particularly those capable of capturing interactions and non-linear relationships between features.

The dataset 
consists of $4$ input variables 
and a single target variable. The values of the input variables ($x_0$, $x_1$, $x_2$, $x_3$) are generated using uniformly distributed random variables, as follows:

\begin{description}
  \item \( X_{0} \) is uniformly distributed on $[0,100]$, 
  \item \( X_{1} \) is uniformly distributed on $[40\pi,560\pi]$, 
  \item \( X_{2} \) is uniformly distributed on $[0,1]$, 
  \item \( X_{3} \) is uniformly distributed on $[1,11]$. 
\end{description}

The target variable \( y \) is generated based on the following non-linear formula involving the input features:

\begin{equation}
y(\mathbf{x}) = \sqrt{x_0^2 + \left( x_1 x_2 - \frac{1}{x_1 x_3} \right)^2} + 
\mathcal{N}(0, \sigma)
\end{equation}
where \( \mathcal{N}(0, \sigma) \) represents a normal distribution with mean $0$ and standard deviation $\sigma$.

This formula introduces non-linearities and interactions between the features, making it a challenging dataset for simple linear models. 
The squared terms and the interaction term involving products and reciprocals of features create a complex response surface.\newline
\noindent \textbf{Make Classification dataset. }
    In Scikit-learn, the {\tt make\_classification} function allows 
    to create synthetic datasets for classification problems. It is extremely useful providing numerous parameters to customize the dataset’s features, which is ideal for benchmarking and testing experiments with machine learning algorithms. We use the two objects that are returned by the function: 
    the feature matrix $X$ and the target vector $y$. The generated features consist of several categories: informative features, linear combinations of informative features, duplicates of some informative features (randomly selected), and random features (noise). The number of features in each category can be specified as parameters of {\tt make\_classification}. This dataset enables evaluating machine learning models in varying scenarios and controlling data complexity, noise, and class balance.


    
\subsection{Architectures}

\begin{figure}[]
\centering
\includegraphics[width=\linewidth]{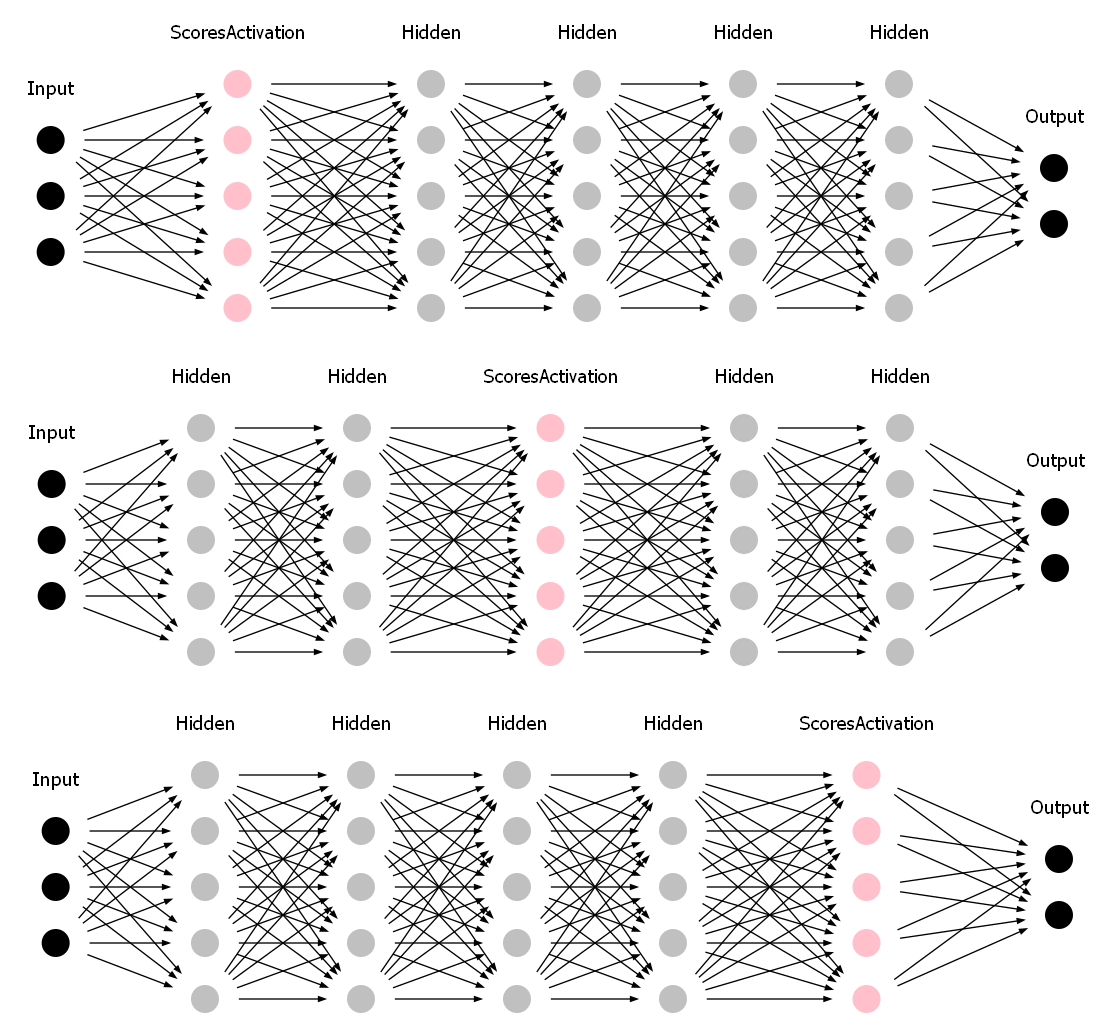}
\caption{ScoresActivation function placement.}
\label{fig:placement}
\end{figure}

\begin{figure*}
\centering
\includegraphics[width=0.55\linewidth]{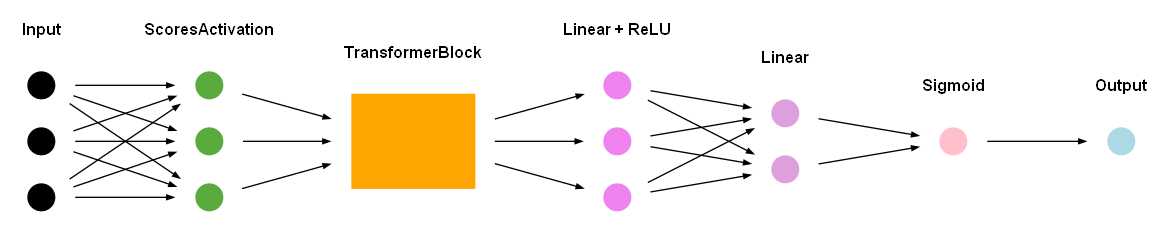}
\caption{Explainable Transformer With ScoresActivation Function Model Architecture.}
\label{fig:XAIScores}
\end{figure*}

\begin{figure*}
\centering
\includegraphics[width=0.55\linewidth]{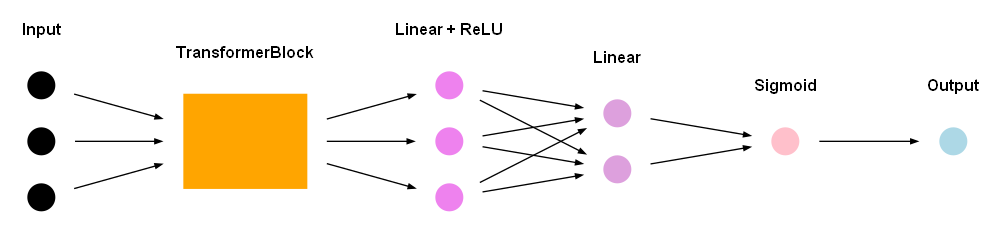}
\caption{Vanilla Model Architecture.}
\label{fig:VanillaModel}
\end{figure*}

The \textbf{Explainable Transformer with ScoresActivation Model} (illustrated in Figure \ref{fig:XAIScores}) extends the  \textbf{Vanilla Model} (illustrated in Figure \ref{fig:VanillaModel}), by incorporating the ScoresActivation function, which introduces an additional layer of interpretability. In the Vanilla Model Architecture, input data is processed directly through a transformer block, followed by fully connected layers that generate the final output. In contrast, the Explainable Transformer with ScoresActivation Model first applies the ScoresActivation function to the input data. Subsequently, the transformed input is passed through the transformer block and fully connected layers, with a final sigmoid activation producing either a binary classification or regression output, depending on the particular dataset. The primary purpose of this model architecture is to evaluate the effectiveness of the ScoresActivation function in enhancing interpretability. Although the current implementation utilizes a transformer-based backbone, the ScoresActivation module is model-agnostic and could be integrated into other neural network architectures to provide similar interpretability benefits. Thus, incorporating ScoresActivation serves to assess its contribution to both model performance and transparency, independent of the specific underlying architecture.

We applied the ScoresActivation function to the first layer of the network. However, it can be implemented on any layer without any adverse effects, as demonstrated in Figure \ref{fig:placement}.

\noindent\textbf{Training}. The datasets were split into training (80\%) and testing (20\%) subsets. Depending on the dataset, we used either Binary Cross-Entropy Loss (BCELoss) or Mean Squared Error (MSE) as the loss function. The Adam optimizer was used for training, with a learning rate of 0.001. All experiments were conducted using Nvidia RTX 3080 and Nvidia Tesla T4 GPUs \cite{gaon}.


\noindent\textbf{Features ranking}. In parallel, we applied the SHAP explanation technique for a direct comparison. We used {\tt Explainer} from the {\tt SHAP} library to compute Shapley values for each feature on a set of 100 random test samples. For each feature, we took the mean absolute SHAP value across those samples as a measure of its overall importance. This yields a SHAP-based global importance ranking of the features, which we treated as a baseline "post-hoc" explanation. We then compared our model's ranking with the SHAP ranking, along with the ground truth.


\begin{table}[h]
\centering

\scriptsize
\renewcommand{\arraystretch}{0.9}
\setlength{\tabcolsep}{3pt}
\begin{tabular}{lccccc}
\toprule
\textbf{Dataset} & \textbf{Ours} & \textbf{SHAP} & \textbf{GT} & \textbf{SHAP=GT} & \textbf{Ours=GT} \\
\midrule
\multirow{5}{*}{Sklearn Clf.}     
    & 1 & 5  & 1 & $\neq$ & \cmark \\
    & 8 & 9  & 2 & $\neq$ & $\neq$ \\
    & 9 & 3  & 3 & \cmark & $\neq$ \\
    & 3 & 8  & 4 & $\neq$ & $\neq$ \\
    & 5 & 1  & 5 & $\neq$ & \cmark \\
\midrule
\multirow{4}{*}{Friedman 1}   
    & 4 & 1  & 4 & $\neq$ & \cmark \\
    & 3 & 4  & 1 & $\neq$ & $\neq$ \\
    & 2 & 2  & 2 & \cmark & \cmark \\
    & 1 & 3  & 5 & $\neq$ & $\neq$ \\
\midrule
\multirow{4}{*}{Friedman 2}   
    & 1 & 3  & 1 & $\neq$ & \cmark \\
    & 4 & 6  & 2 & $\neq$ & $\neq$ \\
    & 7 & 10 & 3 & $\neq$ & $\neq$ \\
    & 5 & 8  & 4 & $\neq$ & $\neq$ \\
\midrule
\multirow{5}{*}{Synth. (Ours)}   
    & 5 & 5  & 5 & \cmark & \cmark \\
    & 2 & 2  & 2 & \cmark & \cmark \\
    & 1 & 1  & 1 & \cmark & \cmark \\
    & 3 & 3  & 3 & \cmark & \cmark \\
    & 4 & 4  & 4 & \cmark & \cmark \\
\bottomrule
\end{tabular}
\caption{Classification Predictions vs. Ground Truth}
\label{tab:ranks}
\end{table}

\section{Numerical Results}

\begin{table}[t]
\centering
\scriptsize
\setlength{\tabcolsep}{4pt}

\begin{tabular}{lccccc}
\toprule
Dataset & Model & Loss & Accuracy & $\Delta$Loss (\%) & $\Delta$Acc. (\%) \\
\midrule
\multirow{2}{*}{Friedman 1}
  & Vanilla & $1.16\times10^{-4}$ & 0.8571 & -- & -- \\
  & Ours    & $3.7\times10^{-8}$  & \textbf{0.9554} & -99.97 & +11.47 \\
\hline
\multirow{2}{*}{Friedman 2}
  & Vanilla & $4.95\times10^{-1}$ & 0.6339 & -- & -- \\
  & Ours    & $3.4\times10^{-6}$  & \textbf{0.9821} & -99.99 & +54.93 \\
\hline
\multirow{2}{*}{Sklearn Clf.}
  & Vanilla & $1.33\times10^{-5}$ & 0.9241 & -- & -- \\
  & Ours    & $4.0\times10^{-10}$ & \textbf{0.9330} & -99.99 & +0.96 \\
\hline
\multirow{2}{*}{Synth. N10}
  & Vanilla & $8.4\times10^{-3}$  & 0.9062 & -- & -- \\
  & Ours    & $3\times10^{-10}$   & \textbf{0.9732} & -99.99 & +7.39 \\
\hline
\multirow{2}{*}{Synth. N16}
  & Vanilla & $3.8\times10^{-4}$  & 0.7545 & -- & -- \\
  & Ours    & $1\times10^{-10}$   & \textbf{0.9777} & -99.99 & +29.58 \\
\hline
\multirow{2}{*}{Breast Cancer}
  & Vanilla & $1.25\times10^{-1}$ & 0.9175 & -- & -- \\
  & Ours    & $1.02\times10^{-5}$ & \textbf{0.9549} & -99.99 & +4.08 \\
\hline
\multirow{2}{*}{Diabetes}
  & Vanilla & $1.02\times10^{-1}$ & 0.8767 & -- & -- \\
  & Ours    & $1.08\times10^{-1}$ & \textbf{0.9082} & +5.88 & +3.59 \\
\hline
\multirow{2}{*}{\shortstack[l]{Diabetes\\(+ rand. feat.)}}
  & Vanilla & $8.65\times10^{-2}$ & 0.9249 & -- & -- \\
  & Ours    & $6.02\times10^{-2}$ & \textbf{0.9427} & -30.40 & +1.92 \\
\bottomrule
\end{tabular}
\caption{Training loss and accuracy after 1000 epochs}
\label{tab:metrics_1000}
\end{table}





\begin{table}[t]
\centering
\scriptsize
\setlength{\tabcolsep}{4pt}
\begin{tabular}{lccccc}
\toprule
Dataset & Model & Loss & Accuracy & $\Delta$Loss (\%) & $\Delta$Acc. (\%) \\
\midrule
\multirow{2}{*}{Friedman 1}
  & Vanilla & $2.25\times10^{-1}$ & 0.7991 & -- & -- \\
  & Ours    & $2.30\times10^{-4}$ & \textbf{0.9554} & -99.90 & +19.56 \\
\hline
\multirow{2}{*}{Friedman 2}
  & Vanilla & $3.78\times10^{-1}$ & 0.7009 & -- & -- \\
  & Ours    & $1.5\times10^{-9}$  & \textbf{0.9821} & -99.99 & +40.12 \\
\hline
\multirow{2}{*}{Sklearn Clf.}
  & Vanilla & $3.7\times10^{-2}$  & 0.9018 & -- & -- \\
  & Ours    & $1.9\times10^{-6}$  & \textbf{0.9286} & -99.99 & +2.97 \\
\hline
\multirow{2}{*}{Synth. N10}
  & Vanilla & $4\times10^{-5}$    & 0.9241 & -- & -- \\
  & Ours    & $9.3\times10^{-9}$  & \textbf{0.9866} & -99.98 & +6.76 \\
\hline
\multirow{2}{*}{Synth. N16}
  & Vanilla & $9.45\times10^{-6}$ & 0.9464 & -- & -- \\
  & Ours    & $1\times10^{-10}$   & \textbf{0.9777} & -99.99 & +3.31 \\
\hline
\multirow{2}{*}{Breast Cancer}
  & Vanilla & $3.71\times10^{-2}$ & 0.8550 & -- & -- \\
  & Ours    & $4.67\times10^{-5}$ & \textbf{0.9766} & -99.87 & +14.22 \\
\hline
\multirow{2}{*}{Diabetes}
  & Vanilla & $9.75\times10^{-2}$ & 0.8963 & -- & -- \\
  & Ours    & $8.33\times10^{-2}$ & \textbf{0.9227} & -14.56 & +2.95 \\
\hline
\multirow{2}{*}{\shortstack[l]{Diabetes\\(+ rand. feat.)}}
  & Vanilla & $8.65\times10^{-2}$ & 0.9149 & -- & -- \\
  & Ours    & $6.02\times10^{-2}$ & \textbf{0.9526} & -30.40 & +4.12 \\
\bottomrule
\end{tabular}
\caption{Training loss and accuracy after 5000 epochs}
\label{tab:metrics_5000}
\end{table}

Across all datasets, illustrated in Tables \ref{tab:metrics_1000} and \ref{tab:metrics_5000}, the model that includes ScoresActivation reduces the training loss (for 5 out of 8 datasets)
and boosts test accuracy by several orders of magnitude, 
while incurring only a modest increase (2\%–5\%) in terms of training time. 
The improvements are most pronounced on the \textbf{Friedman 2} dataset, where accuracy climbs from 0.70 to 0.98 with 5000 epochs and from 0.63 to 0.98 with just 1000 epochs. The ScoresActivation variant also cuts the training loss by two to four orders of magnitude, meaning the network learns to fit the data between 100 times and 10.000 times more precisely. The most significant improvement is a near-100\% reduction, from 3.78\,$\times 10^{-1}$ to 1.5\,$\times 10^{-9}$.




Beyond predictive performance, ScoresActivation produces feature rankings that closely mirror the known target ordering, as observed in Table \ref{tab:ranks}.

ScoresActivation delivers feature rankings essentially "for free". Because scores are produced during the forward pass, generating a full ranking takes only an insignificant amount of milliseconds, whereas post‑hoc SHAP explanations require roughly 12s on the same hardware. This acceleration removes the computational bottleneck of model interpretability and makes faithful, global explanations practical even for large‑scale training.

In every real‑world or benchmark dataset, the ScoresActivation model produces more ground‑truth‑consistent rankings than the post‑hoc SHAP pipeline. The gain is most pronounced on the Friedman 2 problem, where SHAP never selects the true top feature while our method succeeds. Our method is also robust to irrelevant noise. The rows where SHAP diverges from ground truth are often those in which irrelevant or weakly relevant features appear highly ranked.

In short, Table \ref{tab:ranks} shows quantitative evidence that integrating explainability into the training loop yields more faithful and stable global feature rankings than a post‑hoc SHAP analysis without sacrificing performance and while being orders of magnitude faster.

The behavior of scores initialization is illustrated in Figures \ref{fig:scores_ranking_zero_init} and \ref{fig:gt_init}. It does not affect the resulting ranking, demonstrating that initializing with ground truth values does not alter the final ranking produced by the ScoresActivation.



\section{Conclusions}
We presented an architecture that unifies prediction and explanation by incorporating a feature attribution mechanism, ScoresActivation, as an activation function adaptable to any type of model architecture. 

This work addresses the limitation of traditional post-hoc explainability methods, which operate separately from model training and can yield explanations of questionable faithfulness.

By contrast, our model is intrinsically interpretable and as such it learns to output feature importance scores as part of its forward pass, ensuring that the explanations are directly tied to the model’s decision process. Through experiments on both synthetic and real datasets, we demonstrated that our model achieves a higher accuracy than a standard Transformer, all while providing reliable explanations. Crucially, the feature importance rankings showed strong agreement with the ground truth, confirming that our integrated explanations are indeed faithful. Moreover, our approach significantly reduces the computational cost of explanation, making it feasible to use in practice even for large datasets or in time-critical applications. Training the model with the ScoresActivation module proved to be stable and did not require special tuning, indicating that integrating interpretability into the model is a practical and sustainable strategy.

Our work supports the concept that we don’t have to choose between accuracy and interpretability. It is possible to build high-performing models that also explain their reasoning as they run. We hope that this research encourages more exploration into ante-hoc explainability techniques, moving the field toward models that are not only powerful but also transparent and trustworthy by design.

\bibliography{mybibfile}

\end{document}